\relax
\documentclass[letterpaper]{article} % DO NOT CHANGE THIS
\usepackage{aaai21}  % DO NOT CHANGE THIS
\usepackage{times}  % DO NOT CHANGE THIS
\usepackage{helvet} % DO NOT CHANGE THIS
\usepackage{courier}  % DO NOT CHANGE THIS
\usepackage[hyphens]{url}  % DO NOT CHANGE THIS
\usepackage{graphicx} % DO NOT CHANGE THIS
\usepackage{natbib}  % DO NOT CHANGE THIS AND DO NOT ADD ANY OPTIONS TO IT
\usepackage{caption} % DO NOT CHANGE THIS AND DO NOT ADD ANY OPTIONS TO IT

\usepackage{amssymb}
\usepackage{amsmath}
\usepackage{mathtools}
\usepackage{multirow}
\usepackage{booktabs}
\usepackage{xcolor}

\usepackage[switch]{lineno}

\usepackage[ruled,linesnumbered]{algorithm2e}

\usepackage{tablefootnote}
\usepackage{threeparttable}

\newtheorem{theorem}{Theorem}
\newtheorem{proof}{Proof}
\newtheorem{lemma}{Lemma}

\newtheorem{definition}{Definition}

\usepackage{amsmath}
\usepackage{float}
\usepackage{algorithmic}
\usepackage[ruled]{algorithm2e}
\usepackage{graphicx} % DO NOT CHANGE THIS

\title{Enhancing User-Centric Privacy Protection: An Interactive Framework through Diffusion Models and Machine Unlearning}
\author{
    Huaxi Huang\textsuperscript{\rm 1}\thanks{Project lead. Huaxi Huang, Xin Yuan and Qiyu Liao are Co-first authors. Corresponding Author: Dadong Wang. This work was done when Huaxi was a CREC Research Fellow at Data61.}, Xin Yuan\textsuperscript{\rm 1}, Qiyu Liao\textsuperscript{\rm 1}, Dadong Wang\textsuperscript{\rm 1}, Tongliang Liu\textsuperscript{\rm 2}
    \\
}
\affiliations{
    \textsuperscript{\rm 1}Data61, CSIRO, Australia\\
    \textsuperscript{\rm 2}The University of Sydney, Sydney NSW 2006, Australia\\
    
}

\begin{document}
%\linenumbers
\maketitle

\begin{abstract}
In the realm of multimedia data analysis, the extensive use of image datasets has escalated concerns over privacy protection within such data. Current research predominantly focuses on privacy protection either in data sharing or upon the release of trained machine learning models. Our study pioneers a comprehensive privacy protection framework that safeguards image data privacy concurrently during data sharing and model publication. We propose an interactive image privacy protection framework that utilizes generative machine learning models to modify image information at the attribute level and employs machine unlearning algorithms for the privacy preservation of model parameters. This user-interactive framework allows for adjustments in privacy protection intensity based on user feedback on generated images, striking a balance between maximal privacy safeguarding and maintaining model performance. Within this framework, we instantiate two modules: a differential privacy diffusion model for protecting attribute information in images and a feature unlearning algorithm for efficient updates of the trained model on the revised image dataset. Our approach demonstrated superiority over existing methods on facial datasets across various attribute classifications.

%unsupervised 
%the communications between labeled and unlabeled examples are adopted to propagate the label information by adapting the Poisson MBO model, which naturally serves as a more stable and informative classifier than traditional graph-based SSFSL models.
%By constraining the augmented positive pairs to be close while pushing the negative ones to be apart, the proposed unsupervised contrastive transfer scheme implicitly learns the novel-class embedding without the possible over-fitting on the few labeled. 

\end{abstract}

\section{Introduction}
Images inherently contain a wealth of private information, such as gender, race, and age in facial datasets to license plate numbers and vehicle types in car photographs. The extensive use of such images across social networks, government databases, and industrial applications has significantly raised privacy risks, leading to profound public concerns \cite{times2018facebook,liu2021machine}. This growing reliance on image data brings privacy issues to the forefront of global discussions and legislation. Countries around the world have responded by implementing stringent privacy laws, with the European General Data Protection Regulation (GDPR) \cite{GDPR2016a} and the Australian Privacy Act 1988 \cite{no1988privacy} being prime examples. These laws emphasize the safeguarding of ``personal data,'' which includes any information linked to a specific or identifiable individual. Given this definition, images fall under the category of personal data due to their frequent inclusion of personal sensitive attributes such as faces, textual content, and license plates. This situation highlights the urgent need for effective methods to protect image data privacy.

Contemporary multimedia research has focused on dataset publication and model design, with a key challenge being privacy protection. Traditional methods like pixel-level image obfuscation~\cite{fan2018image,fan2019practical} are ineffective against advanced deep learning techniques, while recent approaches using generative adversarial networks (GANs)~\cite{liu2021dp,wen2022identitydp,liposter} suffer from instability and poor image quality. Moreover, despite progress in preventing unauthorized access to sensitive model information, a unified solution addressing privacy in both dataset sharing and model deployment is lacking.
This challenge is particularly relevant for companies providing facial image applications, such as facial recognition and emotion detection services, which require comprehensive privacy solutions to manage user requests for the removal of private data, requiring adjustments to both datasets and models.

Recognizing these theoretical and practical challenges, this paper 
proposes an integrated framework that bridges the gap by effectively managing privacy concerns in image data sharing and secure machine learning model release. 
% This framework not only effectively manages privacy concerns in image data sharing but also excels in the secure release of machine learning models. 
Developed from published datasets, these models are designed to prioritize privacy protection and can efficiently unlearn specific private information upon user requests. This approach ensures a balance between privacy safeguarding and the utility and performance of the data and models. Fig.~\ref{fig2} provides a schematic overview of this framework, illustrating the interactive, user-centric process from customer input to model refinement, underscoring our strategy for dynamic privacy management in both data and models.

Specifically, our framework addresses the privacy preservation challenge in image sharing by introducing differential privacy (DP). To overcome the instability in GAN training and issues with image quality, we have drawn inspiration from diffusion models \cite{ho2020denoising,preechakul2022diffusion} and developed a novel approach to differential image privacy. This method employs diffusion models for both the extraction of intermediate features and the generation of new images. The fundamental advantage of diffusion models lies in their process, which primarily involves adding noise to images and optimizing the subsequent denoising steps, circumventing the adversarial training intrinsic to GAN models and leading to more stable model training. Existing research \cite{ho2020denoising,preechakul2022diffusion} indicates that diffusion models can achieve superior image quality compared to GANs. 
We design the Diffusion Differential Privacy model (Diffusion-DP). By utilizing a diffusion autoencoder model~\cite{preechakul2022diffusion}, we train an auxiliary classifier to disentangle intermediate features, allowing precise control over target face images by injecting DP noises. This enhances the privacy protection of the shared images while maintaining their utility.

To address the challenge of updating models in response to dataset changes, especially in erasing privacy information, our framework incorporates an advanced machine unlearning (AMU) module. This module is designed for scenarios requiring the removal of specific private information from trained models due to dataset updates. Unlike traditional methods of retraining models from scratch, characterized by inefficiency and high costs, our method utilizes cutting-edge machine unlearning techniques \cite{bourtoule2021machine, WarneckePWR23}. We have adapted the machine feature unlearning algorithm \cite{WarneckePWR23} specifically for our task, allowing for rapid fine-tuning of models with minimal data input. 
This enables the models to efficiently unlearn specific information, significantly reducing the resources and time required while ensuring high performance. Integrated with the Diffusion-DP module and the AMU module, our framework offers an efficient solution for adapting models to dataset changes, rigorously safeguarding user privacy, and maintaining the utility and effectiveness of the models.

In summary, this paper makes the following contributions:
\begin{itemize}
    \item We propose a user-centric interactive
image privacy protection framework designed to safeguard user privacy concurrently during image data sharing and model release phases. This framework is interactive, enabling user-driven adjustments for personalized privacy considerations.
    \item We instantiate two specific modules within our broader framework to separately address the processing of private information in data sharing and model updates. The Diffusion-DP module focuses on editing and generating image data with an emphasis on privacy preservation, while the AMU module enables efficient and rapid model adjustments in response to dataset updates, ensuring privacy compliance.
    \item Comprehensive experiments are conducted to evaluate the effectiveness of our proposed framework. The results demonstrate that our approach not only achieves significant improvements in privacy protection but also maintains high utility and performance.
\end{itemize}

The rest of this paper is organized as follows. Section \ref{RL} introduces related works. Section \ref{Method} presents the proposed framework. In Section \ref{exp}, we evaluate the proposed method on widely-used facial datasets. The conclusion is discussed in Section \ref{con}.

\section{Related Work} \label{RL}
\subsection{Privacy Protection Methods for Image Data}

% Traditional privacy-preserving methods \cite{fan2018image,fan2019practical} typically use pixel-level image obfuscation, a technique that falls short against deep learning tools \cite{mcpherson2016defeating}. Furthermore, direct perturbations on original pixels can substantially degrade the image's utility. In response to this, recent research \cite{liu2021dp,wen2022identitydp} proposed using deep GANs for image privacy protection. These methods extrac intermediate features from the encoder, infuse DP noise into these features, and then feed the perturbed features into the GAN's image generator to create a new image. Despite the promising progress these studies suggest, existing challenges remain. For example, GAN training can be unstable and is susceptible to mode collapse in practice \cite{lucic2018gans}. In addition, the quality of the generated images can be further improved.
Traditional privacy-preserving methods like pixel-level image obfuscation are often ineffective against advanced deep learning techniques~\cite{fan2018image,fan2019practical}, as they can significantly degrade the utility of the image and fail to protect against sophisticated attacks~\cite{mcpherson2016defeating}. Recent research has proposed using deep GANs for image privacy protection, where DP noise is infused into intermediate features extracted from the encoder, which are then used by the GAN's image generator to produce new images \cite{liu2021dp,wen2022identitydp}. While this approach offers a more robust solution, it faces challenges such as GAN training instability, susceptibility to mode collapse, and the need for further improvement in the quality of generated images~\cite{lucic2018gans}. 

This challenge extends to contemporary multimedia research, which focuses on both dataset publication~\cite{liu2015faceattributes,CelebAMask-HQ,deng2009imagenet,karras2019style} and specific model design~\cite{He_2016_CVPR,ren2015faster,ronneberger2015u}. Privacy protection remains a critical concern, particularly in the publication and sharing of datasets. Although substantial progress has been made in developing privacy-preserving techniques for both datasets and models, a unified approach that addresses privacy concerns in both domains is still lacking. This gap is particularly evident in practical scenarios faced by companies specializing in facial image applications~\cite{Faceapp,API}, where there is a growing demand for solutions that safeguard privacy while maintaining the functionality of AI services based on facial data, like facial recognition~\cite{deng2019arcface,schroff2015facenet}, age estimation~\cite{pan2018mean,gao2018age,li2022unimodal}, gender identification~\cite{eidinger2014age,mivolo2023}, and emotion detection~\cite{KolliasZ19,zhang2023abaw5,savchenko2023facial}. 
When users request the retraction of their private data to prevent privacy violations, these companies must carefully adjust both datasets and corresponding AI models, underscoring the complexity and importance of a comprehensive privacy-preserving strategy.

\subsection{Machine Unlearning}
In the field of machine learning, both exact and approximate machine unlearning methods have been explored to address the challenge of completely removing the influence of certain data segments from trained models. Exact machine unlearning aims to fully eliminate data influence, often necessitating some degree of model retraining. Early work by Cao \textit{et al.}~\cite{cao2015towards} introduced a heuristic approach to convert machine learning algorithms into a summation format, facilitating the removal of data lineage. Subsequent methods, such as Sharded, Isolated, Sliced, and Aggregated training (SISA) by Bourtoule \textit{et al.}~\cite{bourtoule2021machine}, proposed a sharded, isolated, sliced, and aggregated training approach to enhance retraining efficiency. DeltaGrad~\cite{wu2020deltagrad}, another exact method, accelerates retraining by counteracting the data designated for deletion, though it is limited to specific algorithms and cannot manage mini-batch sizes.

On the other hand, approximate machine unlearning focuses on minimizing the differences between models before and after data removal, aiming to preserve model performance. Notable methods include the certified-removal approach by Guo \textit{et al.}~\cite{guo2020certified}, which uses a Newton step to erase the influence of data points in $L_2$-regularized linear models, and a scrubbing method for deep neural networks proposed by Golatkar \textit{et al.}~\cite{golatkar2020eternal}. These approaches often incorporate differential privacy mechanisms to obscure residual information. Additionally, specialized methods, such as those by Ginart \textit{et al.}~\cite{ginart2019making} for $K$-means, highlight the development of model-specific unlearning techniques. While approximate methods offer privacy safeguards and performance preservation, they face challenges in verifying implementation and aligning with legal requirements like the ``right to be forgotten.''
\section{User-centric Interactive
Image Privacy Protection Framework}\label{Method}

\begin{figure*}[t]
\centering
\includegraphics[width=0.8\linewidth]{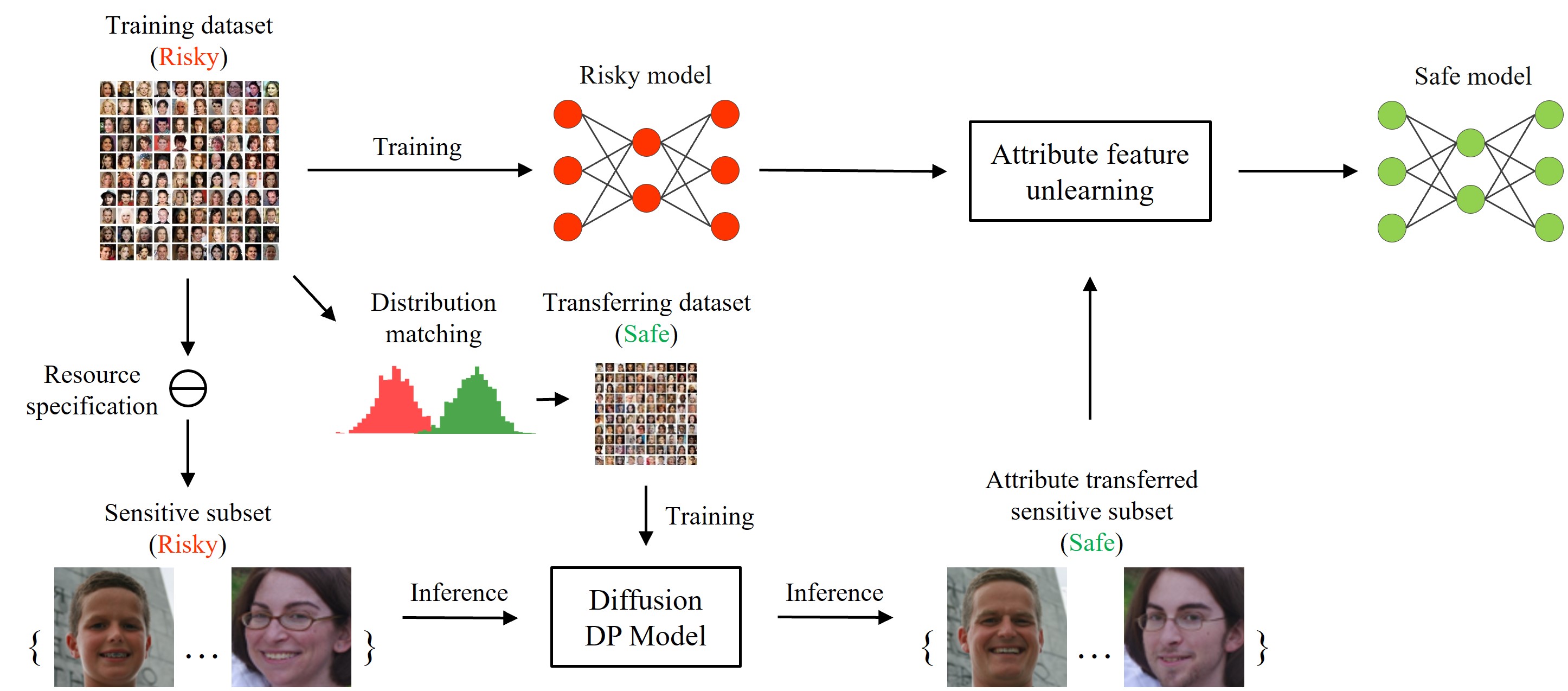}
\caption{A user-centric interactive
image privacy protection framework, which protects sensitive information in image data for machine learning. It transforms a risky training dataset and model into a safe model using attribute feature unlearning and DP techniques, ensuring privacy by modifying sensitive attributes and matching data distributions while balancing privacy protection with model performance through user feedback and adjustments.}
\label{fig2}
\end{figure*}

% The provided framework illustrates a method for creating a safe model from a potentially risky dataset. The process begins with a ``Training dataset (Risky),'' which contains sensitive information. This dataset is used to train a ``Risky model" that might inadvertently learn sensitive attributes. To mitigate this risk, the framework includes a ``Resource specification" step that isolates a ``Sensitive subset (Risky)'' of the data. A ``Distribution matching'' technique then creates a ``Transferring dataset (Safe)" that aligns with the distribution of the sensitive subset without revealing sensitive details.

% The safe dataset is used to train a ``Diffusion DP Model,'' which generates inferences while maintaining privacy. The output undergoes ``Attribute feature unlearning'' to further remove any sensitive information, resulting in a ``Safe model.'' This safe model can then infer data from an ``Attribute transferred sensitive subset (Safe)'' without exposing sensitive attributes. The entire process aims to ensure the privacy and security of sensitive information while still enabling model training and inference.

In this paper, we propose a user-centric interactive
image privacy protection framework that converts a potentially risky model into a safe one by unlearning specific attribute features, as depicted in Fig.~\ref{fig2}. 
The framework begins with a training dataset labeled as ``risky'' due to its inclusion of sensitive information. This dataset is used to train a risky model. 
The process involves identifying a sensitive subset within the training data, guided by a resource specification step, which isolates sensitive attributes or individuals. Concurrently, a transferring dataset, deemed ``safe'', undergoes distribution matching to ensure alignment with the sensitive subset. This allows for the safe transfer of attribute features, producing an ``attribute transferred sensitive subset'' that mitigates privacy risks. 
The core of this framework is the advanced machine unlearning (AMU) module, which employs a refined machine feature unlearning algorithm to efficiently remove sensitive attributes from the model without the need for
costly retraining.

To ensure privacy preservation, the framework integrates a Diffusion-DP module, providing DP guarantees by training on the transformed sensitive subset.
The result is a ``safe model'' that can perform accurate inferences while safeguarding sensitive information. 
This approach not only mitigates the risks associated with sensitive data exposure but also maintains high model performance and utility, making it an efficient solution for adapting to dataset changes while rigorously protecting user privacy.

In what follows, we will provide detailed insights into each module of the framework.

\subsection{Differential Privacy}\label{def}
For an $\left(\epsilon, \delta\right)$-DP mechanism, $\epsilon > 0$ denotes the distinguishable bound of all outputs on adjacent datasets $\mathcal{D}$ and $\mathcal{D}'$ in a dataset{\footnote{Two datasets, $\mathcal{D}$ and $\mathcal{D}'$, are adjacent if $\mathcal{D}'$ can be built by adding or removing a single training example from $\mathcal{D}$.}}, and $\delta$ denotes the probability that the ratio of the probabilities of two adjacent datasets $\mathcal{D}$ and $\mathcal{D}'$ cannot be bounded by $\exp(\epsilon)$ after adding a privacy-preserving mechanism~\cite{mcsherry2007mechanism}. The specific definition of DP is provided as follows.
\begin{definition}[$\left( \epsilon, \delta\right)$-DP~\cite{abadi2016deep}]
A randomized mechanism $\mathcal{ M}$: $\mathcal{ X}\to \mathcal{ R}$ with domain $\mathcal{ X}$ and range $\mathcal{ R}$ satisfies $\left( \epsilon, \delta\right)$-DP, if 
\begin{equation}
	\Pr \left[\mathcal{ M}(\mathcal{ D}) \in \mathcal{ S} \right] \leq e^{\epsilon} \Pr \left[\mathcal{ M}(\mathcal{D}')  \in \mathcal{ S}\right] + \delta,
\end{equation}
for all measurable sets $\mathcal{ S} \subseteq \mathcal{ R}$ and for any two adjacent datasets $\mathcal{D}, \mathcal{D}' \in \mathcal{ X}$.	
\end{definition}

\begin{definition}[$\left(\alpha, \epsilon\right)$-RDP~\cite{mironov2017renyi}]
A randomized mechanism $\mathcal{ M}$: $\mathcal{ X}\to \mathcal{ R}$ with domain $\mathcal{ X}$ and range $\mathcal{ R}$ satisfies $\left(\alpha, \epsilon\right)$-RDP, if 
\begin{equation}
D_{\alpha}(\mathcal{ M}(\mathcal{D})\| \mathcal{ M}(\mathcal{D}')) \leq \epsilon,
\end{equation}
for all measurable sets $\mathcal{ S} \subseteq \mathcal{ R}$ and for any two adjacent datasets $\mathcal{D}, \mathcal{D}' \in \mathcal{ X}$.	
\end{definition}

\begin{lemma}[RDP to $(\epsilon,\delta)-DP$~\cite{mironov2017renyi}]
    If a mechanism $\mathcal{ M}$: $\mathcal{ X}\to \mathcal{ R}$ satisfy $\left( \alpha, \delta\right)$-RDP, it also satisfies $\left( \epsilon + \frac{\log(1/\delta)}{\alpha -1}, \delta\right)$-DP for any $0 < \delta < 1$. Moreover, $\mathcal{ M}$ satisfies pure $\epsilon$-DP.
\end{lemma}

We analyze the sensitivity and privacy performance of the proposed time-varying DP noise perturbation mechanism. We use the $\ell_2$-norm sensitivity, as given by
\begin{equation}\label{eq_sensitivity}
\Delta s = \max_{\mathcal{D},\mathcal{D}'} \left\| s(\mathcal{D}) - s(\mathcal{D}')\right\|,
\end{equation}
where $s(\cdot)$ is a general function in $\mathcal{D}$.	

% \subsection{DP image under Diffusion Mechanism}
\subsection{Diffusion Differential Privacy Model}
To implement DP for images, we consider a diffusion mechanism that generates images while incorporating noise to ensure privacy.
\begin{definition}[Weighted Diffusion Mechanism]\label{def-weighted_diffusion}
    Assume $\mathbf{X},\mathbf{X}' \in {\mathbb R}^d$, let $\mathcal{ M}(\mathbf{X})$ be the mechanism that is a random function taking $X$ as the input, and returning $\mathcal{ M}(\mathbf{X})$ at the $(t+1)$-the iteration:
    \begin{equation}
        X_{t+1}(i) = s_{t+1}(X_t(i)) + {w}^t_c(i) \cdot n_{t+1}(i),
    \end{equation}
    \begin{equation}
        X_{t+1} = Decoder(Encoder(X_t) + {w}^t_c(i) \cdot n_{t+1}(i)),
    \end{equation}
    where $s_{t+1}(\cdot)$ is the contractive map, ${\mathbf{w}}^t_c =[w^t_c(1),\cdots,w^t_c(i),\cdots,w^t_c(d)]  $ is the weighted vector obtained from a $N$-class classifier, and $\mathbf{n}_{t+1} \sim OU (\theta, \rho)$  at the $(t+1)$-th iteration, i.e., $\mathbf{n}_{t} = \mathcal{ N} (e^{-\theta t}x, \frac{\rho^2}{\theta}(1-e^{-2\theta t})\mathbb{I}_d)$. Let $\mathcal{ M}(\mathbf{X}')$ be obtained from $\mathbf{X}'$ under the same mapping. For $\alpha \geq 1$, $\mathcal{M}(\mathbf{X})$ satisfies
    \begin{equation}
                 D_{\alpha}(\mathcal{ M}(\mathbf{X})\| \mathcal{ M}(\mathbf{X}')) \leq \frac{\alpha\|\mathbf{X} - \mathbf{X}' \|^2}{2T \left(\frac{{w}^t_c(i)}{\sum_{i=1}^N {w}^t_c(i)} \cdot \frac{\rho^2}{\theta}(e^{2\theta t}-1)\right)}.
    \end{equation}
\end{definition}

\begin{theorem}[Diffusive DP-Image]
Suppose that an input image $\mathbf{X}_0$, and its corresponding output image $\mathbf{X}_T$. $s(\cdot)$ maps the image into its latent space (or, feature space), and $s'(\cdot)$ is a 1-Lipschitz. 
Let $s(\cdot)$ have global $L_2$-sensitivity $\Delta_s$ and $\mathbf{P} = (P_t)_{t \geq 0}$ be the Ornstein-Uhlenbeck (OU) process with parameters $\theta$ and $\rho$. For any $\alpha >1$ and $t >0 $ the OU mechanism ${\mathcal{M}}_t^s(\mathbf{X}_t) = P_t (s(\mathbf{X}_t))$ satisfies $\left(\alpha, \frac{\alpha \theta \Delta_s^2}{2T \left(\frac{{w}^t_c(i)}{\sum_{i=1}^N {w}^t_c(i)} \rho^2 (e^{2 \theta t -1 })\right)}\right)$-RDP.
\end{theorem}

\begin{proof}
According to {\bf{Definition~\ref{def-weighted_diffusion}}}, we have $\mathcal{ M}(s(\mathbf{X}))$ for $\alpha \geq 1$ satisfies
\begin{subequations}\label{eq-divergence}
\begin{align}
   D_{\alpha}\left(\mathcal{ M}(s(\mathbf{X}))\| \mathcal{ M}(s(\mathbf{X}'))\right) 
   & \leq \frac{\alpha\|s(\mathbf{X}) - s(\mathbf{X}') \|^2}{2T\left(\frac{{w}^t_c(i)}{\sum_{i=1}^N {w}^t_c(i)} \cdot \frac{\rho^2}{\theta}(e^{2\theta t}-1)\right)} \label{eq-divergence a}\\
   & \leq \frac{\alpha \max_{\{\bf{X},\bf{X}'\}} \|s(\mathbf{X}) - s(\mathbf{X}') \|_2^2}{2T \left(\frac{{w}^t_c(i)}{\sum_{i=1}^N {w}^t_c(i)} \cdot \frac{\rho^2}{\theta}(e^{2\theta t}-1)\right)} \label{eq-divergence b}\\
   & = \frac{\alpha \Delta^2_s}{2T \left(\frac{{w}^t_c(i)}{\sum_{i=1}^N {w}^t_c(i)} \cdot \frac{\rho^2}{\theta}(e^{2\theta t}-1)\right)}, \label{eq-divergence c}
\end{align}
\end{subequations}
where \eqref{eq-divergence c} is obtained based on \eqref{eq_sensitivity}.

Based on the definition of $(\alpha,\epsilon)$-RDP, we have $ D_{\alpha}(\mathcal{ M}(\mathbf{X})\| \mathcal{ M}(\mathbf{X}')) \leq \epsilon$. Therefore, the OU mechanism ${\mathcal{M}}_t^s(\mathbf{X}_t) = P_t (s(\mathbf{X}_t))$ satisfies $\left(\alpha, \frac{\alpha \theta \Delta_s^2}{2T \left(\frac{{w}^t_c(i)}{\sum_{i=1}^N {w}^t_c(i)} \rho^2 (e^{2 \theta t -1 })\right)}\right)$-RDP.

\end{proof}

\begin{algorithm}[!t]
\caption{Diffusion-DP algorithm}\label{alg:cap}
\LinesNumbered
	\textbf{Parameters:}{~Noise coefficient $\sigma$, and number of iterations $T$.}\\
	\KwIn{The original image $X_t$.}
	\KwOut{Output Privacy-preserving image $X_T$.}
	\While {$0 \leq t < T$}{
	$\mathbf{z}_t = s(\mathbf{X}_t)$ \\
	$\mathbf{z}_{t+1} = s'(\mathbf{z}_t) + \mathbf{w}^t_c\cdot \mathbf{n}_t, \; \mathbf{n}_t \sim {\text{OU}}(\theta,\rho)$\\
	$\mathbf{X}_{t+1} = g(\mathbf{z}_{t+1})$}
% 	some description\;
\end{algorithm}

\subsection{Advanced Machine Unlearning}

This section presents our methodology for unlearning specific attributes from input images using first-order and second-order optimization techniques, aiming to remove the influence of specific attributes from the model while preserving the overall performance of other features.

\subsubsection{First-Order Update}
To unlearn data, we aim to find an update \(\Delta(\mathbf{X}, \tilde{\mathbf{X}})\) for our model \(w^*\),
where $\mathbf{X}$ and $\tilde{\mathbf{X}}$ are the original data and its perturbed version. If the loss \({\mathcal{L}}\) is differentiable, we compute the \textit{first-order update} as:
\begin{equation}
\Delta(\mathbf{X}, \tilde{\mathbf{X}}) = -\eta \left( \sum_{\tilde{x} \in \tilde{\mathbf{X}}} \nabla_w {\mathcal{L}}(\tilde{x}, w^*) - \sum_{x \in \mathbf{X}} \nabla_w {\mathcal{L}}(x, w^*) \right),
\end{equation}
where \(\eta\) is the unlearning rate. This update shifts the model parameters to minimize the loss on \(\tilde{x}\) while removing the information in \(x\). It differs from gradient descent by moving the model based on the gradient \textit{difference} between original and perturbed data. The gradients can be computed in \(\mathcal{O}(p)\), with $p$ being the number of parameters in the learning model~\cite{WarneckePWR23}.

\subsubsection{Second-Order Update}
The unlearning rate \(\eta\) can be eliminated if \({\mathcal{L}}\) is twice differentiable and strictly convex. The influence of a single data point is approximated by:
\begin{equation}
\left. \frac{\partial w^*_{\phi, x \rightarrow \tilde{x}}}{\partial \phi} \right|_{\phi=0} = -H_{w^*}^{-1} \left( \nabla_w {\mathcal{L}}(\tilde{x}, w^*) - \nabla_w {\mathcal{L}}(x, w^*) \right),
\end{equation}
where \(H_{w^*}^{-1}\) is the inverse Hessian of the loss at \(w^*\). This leads to a linear approximation, as given by
\begin{equation}
w^*_{x \rightarrow \tilde{x}} \approx w^* - H_{w^*}^{-1} \left( \nabla_w {\mathcal{L}}(\tilde{x}, w^*) - \nabla_w {\mathcal{L}}(x, w^*) \right).
\end{equation}
Extending this to multiple data points gives the \textit{second-order update}:
\begin{equation}
\Delta(\mathbf{X}, \tilde{\mathbf{X}}) = -H_{w^*}^{-1} \left( \sum_{\tilde{x} \in \tilde{X}} \nabla_w {\mathcal{L}}(\tilde{x}, w^*) - \sum_{x \in X} \nabla_w {\mathcal{L}}(x, w^*) \right).
\end{equation}

This update does not require parameter calibration, as it derives from the inverse Hessian of the loss function. The second-order update is preferred for unlearning in models with strongly convex and twice differentiable loss functions, and can be easily calculated with common machine learning frameworks.

\section{Experiment}\label{exp}

% \subsection{Dataset}
We verify the effectiveness of our proposed method on the two manually corrupted datasets: CelebA-HQ dataset \cite{CelebAMask-HQ} and CelebA dataset \cite{liu2015faceattributes}.
CelebA-HQ contains 30k high-resolution celebrity images with 40 attributes. We use this CelebA-HQ dataset to train our diffusion model. 
CelebA has 202,599 images with 40 same attributes as CelebA-HQ. We use this dataset to test our proposed method. 
% Official Training set: 162,770; Validation set: 19,867; Testing set: 19,962.

\subsection{Comparison Methods}
We compare our method against several SOTA methods. The image-manipulated methods involve processing the input images (224x224 pixels) using different techniques: (i) Random central removal of 100x100 pixels, which involves randomly removing a central portion of the image; (ii) Whole image mosaic, which applies a mosaic filter to the entire image, blurring detailed features; and (iii) our proposed diffusion-based method, which leverages diffusion processes to selectively alter specific attributes in the images. For fine-tuning methods, we implement three unlearning approaches: (i) Retraining, which involves training the model from scratch on a modified dataset; (ii) First-order unlearning, which adjusts the model parameters using gradient descent based on the attribute-specific loss; and (iii) Second-order unlearning, our proposed method, which employs both gradient and curvature information to more effectively remove specific attributes from the model. 

\subsection{Evaluation Metrics}

In our experiments, conducted on the CelebA dataset, we employed several evaluation metrics to assess the performance of attribute modification and image realism. Specifically, we focused on two primary metrics: Attribute Classification Accuracy (ACC) and Structural Similarity Index (SSIM).

For the assessment of attribute modification, we utilized the ACC. In this approach, we trained a linear binary classifier on the training set for each of the 40 attributes. This allowed us to evaluate the accuracy with which each attribute was correctly modified or maintained. The effectiveness of the attribute modifications across the entire image was further quantified by calculating the average ACC across all 40 attributes. This average ACC provides a comprehensive measure of the system's performance in handling multiple attributes simultaneously, ensuring that modifications are not only accurate but also consistent across the dataset.

In addition to ACC, we also assessed the realism of the generated images using the SSIM. Following the methodologies presented in \cite{hinojosa2021learning, liu2021dp}, SSIM was chosen due to its ability to measure the perceptual quality of images. SSIM evaluates the similarity between two images based on luminance, contrast, and structural information, producing a similarity score that ranges from -1 to 1. Higher SSIM values indicate a greater degree of similarity between the images, reflecting the extent to which the modified images retain their visual fidelity and align with human visual perception. Conversely, lower SSIM values suggest a decline in image quality or realism.

Together, these metrics provide a robust framework for evaluating both the technical accuracy of attribute modifications and the visual quality of the resulting images. The combination of ACC and SSIM ensures that our approach is not only effective in modifying attributes but also capable of generating realistic and visually pleasing images, which is critical for applications that require high levels of visual authenticity.

\subsection{Experimental Results}

\begin{table*}[!t]
\small
\centering
\caption{Comparison of different image manipulation methods and their impact on attribute classification performance.}
\label{tab-noise-combined}
\begin{tabular}{c|c|c|c|c|c|c|c}
\hline
\textbf{Change} & \textbf{Methods} & \textbf{Average} & \textbf{Young} & \textbf{Attract} & \textbf{Bald} & \textbf{Bangs} & \textbf{Blurry} \\
\hline
\multirow{1}{*}{} & \textbf{Clean} & 89.75 & 85.81 & 79.58 & 98.59 & 94.78 & 95.73 \\
\hline
\multirow{3}{*}{\textbf{1\%}} & \textbf{Mosaic} & 89.06(-0.69) & 82.38(-3.43) & 74.1(-5.48) & 98.53(-0.06) & 94.6(-0.18) & 95.79(+0.06) \\
& \textbf{Random} & 89.31(-0.44) & \textbf{84.32(-1.49)} & 76.04(-3.54) & 98.52(-0.07) & 94.55(-0.23) & \textbf{95.9(+0.17)} \\
& \textbf{Ours} & \textbf{89.82(+0.07)} & 84.13(-1.68) & \textbf{77.8(-1.78)} & \textbf{98.59(0.00)} & \textbf{94.72(-0.06)} & 95.78(+0.05) \\
\hline
\multirow{3}{*}{\textbf{3\%}} & \textbf{Mosaic} & 88.96(-0.79) & 82.94(-2.87) & 74.84(-4.74) & 98.57(-0.02) & \textbf{94.56} & \textbf{95.67(-0.06)} \\
& \textbf{Random} & \textbf{89.28(-0.47)} & \textbf{84.16(-1.65)} & 75.56(-4.02) & 98.58(-0.01) & \textbf{94.56(-0.22)} & 95.6(-0.13) \\
& \textbf{Ours} & 89.07(-0.68) & 83.53(-2.28) & \textbf{77.6(-1.98)} & \textbf{98.59(0.00)} & \textbf{94.56(-0.22)} & 95.66(-0.07) \\
\hline
\multirow{3}{*}{\textbf{5\%}} & \textbf{Mosaic} & 88.96(-0.79) & 79.94(-5.87) & 78.84(-0.74) & \textbf{98.57(-0.02)} & \textbf{94.56(-0.22)} & \textbf{95.67(-0.06)} \\
& \textbf{Random} & 88.91(-0.84) & 82.87(-2.94) & 71.07(-8.51) & \textbf{98.57(-0.02)} & 81.97(-12.81) & 95.42(-0.31) \\
& \textbf{Ours} & \textbf{89.05(-0.7)} & \textbf{82.95(-2.86)} & \textbf{79.63(+0.05)} & 98.4(-0.19) & 94.09(-0.69) & 95.42(-0.31) \\
\hline
\end{tabular}
\end{table*}

Table~\ref{tab-noise-combined} provides an extensive comparison of various image manipulation methods and their impact on attribute classification performance under different levels of data change (i.e., 1\%, 3\%, and 5\%). The methods compared include `Mosaic', `Random', and `Ours', with each method's performance evaluated across multiple attributes such as \texttt{Young}, \texttt{Arched}, \texttt{Bald}, \texttt{Bangs}, and \texttt{Blurry}.

Table \ref{tab-noise-combined} presents the results for attributes categorized under different percentages of data changes. 
Notably, in the case of 1\% data change, 
our method achieves the highest average accuracy of 89.82\% across 40 attributes, outperforming both the Mosaic' (82.38\%) and `Random' (84.32\%) methods. 
Our method also shows strong performance in specific attributes, such as \texttt{Attract} and \texttt{Bald}, indicating its effectiveness in handling minimal data changes.
As the data change increases to 3\% and 5\%, the `Ours' method consistently demonstrates superior performance, maintaining high accuracy in most attributes. 
For example, under 3\% data change, our method reaches an accuracy of 89.07\%, compared to 88.96\% and 89.28\% for `Mosaic' and `Random', respectively.
This trend continues under the 5\% data change scenario, where our method achieves an accuracy of 89.05\%, again outperforming Mosaic (88.96\%) and Random (88.91\%).

% In Table \ref{tab-noise-combined}, the performance analysis extends to other attributes like \texttt{Big\_Lips}, \texttt{Big\_Nose}, and \texttt{Blurry}. The `Ours' method shows a marked improvement in most categories, particularly in higher data change scenarios (5\%), where it outperforms both `Mosaic' and `Random' methods significantly. For instance, in the \texttt{Big\_Nose} attribute under 5\% data change, `Ours' method achieves an accuracy of 80.86\%, surpassing both `Mosaic' and `Random'.

Overall, the results indicate that our method is particularly robust and adaptable, offering high accuracy even as the data change percentage increases. This consistent performance across different attributes and levels of data manipulation highlights the method's potential for robust attribute classification in dynamic datasets. Such reliability is crucial for real-world applications where data may vary significantly, and accurate attribute classification is essential. 
Our proposed method's superior results in most categories, especially under higher data changes, highlight its effectiveness and utility in improving attribute classification performance.

\begin{table*}[t]
\small
\centering
\caption{Comparison of different unlearning methods and their impact on attribute classification performance.}
\label{tab-unlearn-combined}
\begin{tabular}{c|c|c|c|c|c|c|c}
\hline
\textbf{Change} & \textbf{Methods} & \textbf{Average} & \textbf{Young} & \textbf{Attract} & \textbf{Bald} & \textbf{Bangs} & \textbf{Blurry} \\
\hline
\multirow{3}{*}{\textbf{Clean}} & \textbf{Retraining} & 89.75 & 85.81 & 79.58 & 98.59 & 94.78 & 95.73 \\
& \textbf{First\_order} & \textbf{89.03(-0.72)} & 78.07(-7.74) & 67.67(-11.91) & 98.31(-0.28) & 91.13(-3.65) & 95.03(-0.70) \\
& \textbf{Second\_order} & 89.01(-0.74) & \textbf{79.65(-6.16)} & \textbf{72.65(-6.93)} & \textbf{98.49(-0.10)} & \textbf{91.44(-3.34)} & \textbf{95.07(-0.66)} \\
\hline
\multirow{3}{*}{\textbf{3\%}} & \textbf{Retraining} & 89.07(-0.68) & 83.35(-2.46) & 76.98(-2.60) & 98.46(-0.13) & 94.47(-0.31) & 95.60(-0.13) \\
& \textbf{First\_order} & 89.00(-0.75) & 82.80(-3.01) & 60.62(-18.96) & 98.07(-0.52) & 92.53(-2.25) & 95.14(-0.59) \\
& \textbf{Second\_order} & \textbf{89.01(-0.74)} & \textbf{83.26(-2.55)} & \textbf{68.01(-11.57)} & \textbf{98.29(-0.30)} & \textbf{94.00(-0.78)} & \textbf{95.16(-0.57)} \\
\hline
\multirow{3}{*}{\textbf{5\%}} & \textbf{Retraining} & 89.01(-0.74) & 83.91(-1.90) & 76.96(-2.62) & 98.37(-0.22) & 93.94(-0.84) & 95.62(-0.11) \\
& \textbf{First\_order} & 88.66(-1.09) & 81.18(-4.63) & 67.93(-11.65) & 98.01(-0.58) & 92.01(-2.77) & 95.1(-0.63) \\
& \textbf{Second\_order} & \textbf{89.02(-0.73)} & \textbf{81.27(-4.54)} & \textbf{68.18(-11.40)} & \textbf{98.23(-0.36)} & \textbf{93.00(-1.78)} & \textbf{95.15(-0.58)} \\
\hline
\multirow{3}{*}{\textbf{10\%}} & \textbf{Retraining} & 89.05(-0.70) & 82.95(-2.86) & 73.98(-5.60) & 98.40(-0.19) & 94.09(-0.69) & 95.42(-0.31) \\
& \textbf{First\_order} & 88.64(-1.11) & 79.10(-6.71) & 61.10(-18.48) & 98.03(-0.56) & \textbf{91.19(-3.59)} & 95.05(-0.68) \\
& \textbf{Second\_order} & \textbf{88.99(-0.76)} & \textbf{79.54(-6.27)} & \textbf{61.49(-18.09)} & \textbf{98.04(-0.55)} & 91.15(-3.63) & \textbf{95.13(-0.60)} \\
\hline
\end{tabular}
\end{table*}

Table \ref{tab-unlearn-combined} provides a detailed comparison of various unlearning methods and their impact on attribute classification performance under different levels of data change (3\%, 5\%, and 10\%). The methods evaluated include `Retraining', `First\_order', and `Second\_order', with  their performance assessed across multiple attributes such as \texttt{Young}, \texttt{Arched}, \texttt{Bald}, \texttt{Bangs}, and \texttt{Blurry}. The differences from the `clean' baseline are highlighted, with all selected attributes having an initial accuracy of less than 90\%.

% Table \ref{tab-unlearn-combined} shows the results for attributes categorized under different percentages of data changes. 
In the clean data scenario, the `Retraining' method method achieves the highest average accuracy of 89.75\% across 40 attributes, followed closely by `First\_order' (89.03\%) and `Second\_order' (89.01\%). As the data change increases to 3\%, 5\%, and 10\%, the `First\_order' and `Second\_order' methods exhibit competitive performance, often surpassing `Retraining' in specific attributes. 
For instance, under a 3\% data change, `Second\_order' attains notable accuracies in attributes like \texttt{Bald} (98.29\%) and \texttt{Bangs} (94.00\%), with values close to those of the clean scenario, demonstrating robustness against minor data alterations.
As the data change reaches 5\% and 10\%, the `First\_order' and `Second\_order' methods continue to demonstrate robust performance. 
Specifically, at a 10\% data change, `Second\_order' maintains an average accuracy of 88.99\%, excelling in challenging attributes like \texttt{Bald} (98.04\%) and \texttt{Bangs} (91.15\%). 
The `First\_order' method, although slightly lower in accuracy, still achieves a respectable average accuracy of 88.64\%, showing reasonable adaptability in high data change scenarios.
The method's ability to maintain high accuracy under significant data changes underscores its robustness and effectiveness.

Overall, these results demonstrate the efficacy of the `First\_order' and `Second\_order' unlearning methods in maintaining high attribute classification performance across varying levels of data change. The `Second\_order' method, in particular, shows superior adaptability and accuracy in more challenging scenarios, suggesting its potential as a reliable unlearning technique for robust attribute classification in dynamic and diverse datasets. These insights are valuable for selecting appropriate unlearning methods to enhance attribute classification performance in practical applications.
\begin{table}[t]
\caption{Image generation quality on CelebA-HQ~\cite{CelebAMask-HQ}. }
\label{tab:my-table}
\small
\begin{tabular}{l|c}
\hline
\textbf{Model}                          & $\textbf{SSIM}\uparrow$ \\ \hline
StyleGAN2 ($\mathcal{W}$) \cite{karras2020analyzing}  & 0.677                                     \\
StyleGAN2 ($\mathcal{W}$+) \cite{karras2020analyzing}  & 0.827                                     \\
VQ-GAN \cite{esser2021imagebart}                       & 0.782                                     \\
DDIM (T=100, $128^2$) \cite{songdenoising}        & 0.917                                     \\
Diff-AE (T=100, $128^2$) \cite{preechakul2022diffusion}                                        & \textbf{0.991}                            \\
\textbf{Ours} (T=100, $128^2$)                            & { 0.951}                               \\ \hline
\end{tabular}
\end{table}

Table~\ref{tab:my-table} compares the performance of various image generation models on the CelebA-HQ dataset, measured by the SSIM metric.
%, which assesses the preservation of image quality after attribute modifications. 
%StyleGAN2 using the $\mathcal{W}$ space achieves a moderate SSIM of 0.677, while the $\mathcal{W}$+ space improves this to 0.827. VQ-GAN shows a slightly lower performance with an SSIM of 0.782. DDIM achieves a higher score of 0.917, demonstrating effective image quality preservation, and 
Diff-AE leads with a near-perfect SSIM of 0.991, indicating exceptional retention of image details. Our method, with an SSIM of 0.951, ranks second, showcasing its strong capability to preserve image realism, highlighting its competitiveness and effectiveness in generating high-quality images after attribute modifications.

% Our model demonstrates competitive performance in generating high-quality images, validating its effectiveness in preserving image realism after attribute modifications.

\begin{table}[t]
\small
\caption{Training times for different methods.}
\label{tab-time}
\centering
\begin{tabular}{c|c}
\hline
\textbf{Methods}     & \textbf{Training Time} \\
\hline
Clean               & 13322s             \\
Retraining          & 15185s             \\
First\_order       & 174.41s            \\
Second\_order       & 506.99s            \\
\hline
\end{tabular}
\end{table}
Table~\ref{tab-time} summarizes the training times for different unlearning methods. 
The `Clean' method takes 13,322 seconds, while the `Retraining' method, which likely involves comprehensive retraining, required the most time at 15,185 seconds. In contrast, the `First\_order' and `Second\_order' methods drastically reduce the training times to 174.41 seconds and 506.99 seconds, respectively. 
Notably, the `Second\_order' method, while not as fast as `First\_order', still offers a substantial reduction in time compared to `Clean' and `Retraining'. This additional computational time for the `Second\_order' method suggests it may involve a more detailed unlearning process, potentially leading to better accuracy, as indicated in Table~\ref{tab-unlearn-combined}. 
This efficiency, coupled with the possibility of improved accuracy, makes the `Second\_order' method particularly valuable for applications where balancing computational resources and maintaining high performance is critical. The sharp decrease in training times with both `First\_order' and `Second\_order' methods demonstrates their potential utility in scenarios necessitating quick updates or modifications, while the slight increase in computational time for the `Second\_order' method could indicate a trade-off that favors accuracy without significantly compromising efficiency.

\section{Conclusion}\label{con}
This paper introduces a privacy protection framework for image data, designed to address challenges in both data sharing and model publication. By integrating a diffusion model for attribute modification with machine unlearning algorithms to secure model parameters, our approach successfully balances the dual objectives of privacy protection and model performance. The framework’s user-interactive design enables customized privacy settings, ensuring that high-quality image generation and precise attribute modifications are maintained. Our results demonstrate that this framework not only preserves image realism but also enhances accuracy across facial datasets, representing a significant advancement in privacy-preserving machine learning.

% \section*{Acknowledgment}

% The authors greatly appreciate the financial support from .

% \section*{Acknowledgment}

% The authors greatly appreciate the financial support from the Rail Manufacturing Cooperative Research Centre (funded jointly by participating rail organizations and the Australian Federal Government’s Business-Cooperative Research Centres Program) through Project R3.7.3 - Rail infrastructure defect detection through video analytics.

%\newpage
\small
\bibliography{aaai}

\end{document}